\documentclass[11pt]{article}
\usepackage[letterpaper,margin=1in]{geometry}

\usepackage[dvipsnames]{xcolor}

\usepackage[utf8]{inputenc} %
\usepackage[T1]{fontenc}    %
\usepackage{url}            %
\usepackage{booktabs}       %
\usepackage{amsfonts,amssymb}       %
\usepackage{nicefrac}       %
\usepackage{microtype}      %
\usepackage{subcaption}
\usepackage{graphicx}
\usepackage[numbers]{natbib}
\usepackage{enumerate}
\usepackage[textfont=sl,labelfont=bf,margin={0.3cm,0.5cm}]{caption}

\usepackage{amsmath}
\usepackage{amsthm}
\usepackage{common}
\usepackage{macros}

\usepackage{soul}

\DeclareMathOperator{\tr}{tr}

\newcommand{\E}{\mathbb{E}}

\newcommand{\D}{\mathcal{D}}

\newcommand{\X}{\mathcal{X}}
\newcommand{\Y}{\mathcal{Y}}

\newcommand{\x}{\mathbf{x}}

\newcommand{\y}{\mathbf{y}}

\newcommand{\eg}{\textit{e.g.}}
\newcommand{\ie}{\textit{i.e.}}

\newcommand{\tdet}{\text{\sc det}}

\newcommand{\utext}[2]{\underbrace{#1}_{\text{#2}}}
  
\let\d\relax
\DeclareMathOperator{\d}{d\!}

\newtheorem{lemma}[thm]{Lemma}
\newtheorem{ProvedTheorem}[thm]{Theorem}

\def\cuta#1{\textcolor{PineGreen}{#1}}
\def\cuta#1{}

\def\cut#1{ {\color{green}#1}}
\def\cut#1{}

\newif\ifhideproofs
\hideproofsfalse

\ifhideproofs
\usepackage{environ}
\NewEnviron{hide}{}

\fi

\begin{document}

\title{The Information Complexity of Learning Tasks,\\ their Structure and their Distance%
}
\author{Alessandro Achille$^*$ \& Giovanni Paolini$^\dagger$ \& Glen Mbeng$^\ddag$  \& Stefano Soatto$^*$ \\~\\ $^*$Department of Computer Science, University of California, Los Angeles\\ $^\dagger$Department of Mathematics, Scuola Normale Superiore, Pisa -- Italy\\ $^\ddag$Department of Physics, Scuola Internazionale di Studi Superiori Avanzati, \\ Istituto Nazionale di Fisica Nucleare, Trieste -- Italy \\
 {\tt \small achille@cs.ucla.edu \ giovanni.paolini@sns.it \ g.mbeng@sissa.it \ soatto@ucla.edu}}
\date{}

\maketitle

\begin{abstract}
We introduce an asymmetric distance in the space of learning tasks, and a framework to compute their complexity.
These concepts are foundational for the practice of transfer learning,  whereby a parametric model is pre-trained for a task, and then fine-tuned for another. The framework we develop is non-asymptotic, captures the finite nature of the training dataset, and allows distinguishing learning from memorization.
It encompasses, as special cases, classical notions from Kolmogorov complexity, Shannon, and Fisher Information. However, unlike some of those frameworks, it can be applied to large-scale models and real-world datasets.
Our framework is the first to measure complexity in a way that accounts for the effect of the optimization scheme, which is critical in Deep Learning. 
\end{abstract}

\section{Introduction}

The widespread use of Deep Learning is due in part to its flexibility: One can {\em pre-train} a deep neural network for a task, say finding cats and dogs in images, and then {\em fine-tune} it for another, say detecting tumors in a mammogram, or controlling a self-driving vehicle. Sometimes it works. So far, however, it has not been possible to {\em predict} whether such a ``transfer learning'' practice will work, and how well. 
Even the most fundamental questions are still unanswered: How far are two tasks? In what space do learning tasks live? What is the complexity of a task? How difficult it is to transfer from one task to another?
In this paper, we lay the foundations for answering these questions. 
The contributions of this paper can be summarized as follows:

 \begin{enumerate}
     \item  We introduce a {\em distance between learning tasks}  (\Cref{sec:distance}), where tasks are represented by finite datasets of input data and discrete output classes (labels). Each task can have a different number of labels and a different number of samples. The distance behaves properly with respect to composition and inclusion relations (\Cref{cor:union-distance}) and can handle corner cases, such as learning random labels, where other distances such as Kolmogorov's fail (\Cref{ex:random-label-complexity}). The distance is asymmetric by design, as it may be easier to learn one task starting from the solution to another, than the reverse (\Cref{def:kolmogorov-distance}).
     
     \item Some learning tasks are more difficult than others, so {\em complexity} plays a key role in their distance (\Cref{sec:task-complexity}). We define a new notion of complexity of a learning task that captures as special cases Kolmogorov complexity, Shannon Mutual Information, and the Fisher Information of the parameters of a model (\Cref{prop:kolmogorov-complexity-equivalence}, \Cref{prop:shannon-complexity-equivalence}, \Cref{prop:fisher-complexity-equivalence}). It leverages and extends the classical notions of Structure Function and minimal sufficient statistics. It also provides a  new interpretation to the PAC-Bayes generalization bound (\Cref{sec:pac-bayes}).
     
     \item We show how a parametric function class, such as deep neural networks trained with stochastic gradient descent, can be employed to measure the complexity of a learning task (\Cref{sec:information-in-the-weights}). Such complexity is related to the ``information in the parameters'' of the network. Experiments  show that the resulting distance correlates with the ease of transfer learning by fine-tuning a pre-trained model (\Cref{sec:experiments}).

     \item We show that the asymmetric distance between tasks is a lower bound to the cost of transfer learning using deep neural networks. We introduce the notion of {\em accessibility} (or \emph{reachability}) of a learning task, and illustrate examples of tasks that, while close to each other, are not accessible from one another. Accessibility depends on the global geometry of the space of tasks and on the optimization scheme employed. We conjecture that tasks are accessible through a local optimization scheme when there exists a continuous path of minimizers of the Structure Function (\Cref{sec:task-accessibility}).
     \end{enumerate}

\noindent Our notion of complexity of a learning task formalizes intuitions regarding overfitting and generalization, and differentiates \emph{learning} from \emph{memorization}. It is, to the best of our knowledge, the first complexity measure to take into account the training algorithm, rather than just its asymptotic outcome. It is also the first to enable practical computation for models with millions of parameters, as those currently in use in real applications.

\subsubsection*{Organization of the paper}

The four main sections (3, 4; 5, 6) are organized as a matrix: Along the rows, Sections 3 and 4 use the language of Kolmogorov to introduce the notions of complexity, structure, and distance of learning tasks, whereas Sections 5 and 6 introduce our more general framework, with classical concepts from Kolmogorov, Shannon and Fisher emerging as special cases. Along the columns, Sections 3 and 5 deal with {\em complexity} and the associated notion of information of a single task, whereas Sections 4 and 6 deal with the {\em distance} between tasks.

\subsubsection*{Related work}

This work tackles the theoretical foundations of transfer learning, an active area of research too vast to review here (see \cite{csurka2017domain} for a recent survey). More specifically, \cite{zamir2018taskonomy} computes pairwise distances between tasks by explicitly  fine-tuning models for each pair. This method does not scale to a large number of tasks and only provides pairwise distances. On the other hand, \cite{achille2019task2vec} uses the Fisher Information Matrix to construct a linear embedding that enables to {\em predict} the effect of fine-tuning without explicitly performing the computation. This provides empirical validation to our framework in \Cref{sec:information-in-the-weights}. The existence of tasks that are similar according to most distances, yet one cannot ``reach'' one from another, has been observed empirically in  \cite{achille2017critical}.  This observation  motivates our exploration of the dynamics of learning in \Cref{sec:generalized-distance}.
The notion of reachability in dynamical systems has been studied extensively for known models, whereas in our case the model changes as the training progresses.
Numerical methods for reachability computation \cite{lygeros1999controllers} are restricted to low-dimensional spaces,  whereas typical models in use in the practice can have millions of parameters. 
Our work leverages classical results from Kolmogorov's complexity theory \cite{vereshchagin2004kolmogorov}, classical statistical inference, and information theory, and relates to recent theoretical frameworks for Deep Learning including the Information Bottleneck Principle \cite{tishby2000information}, with important differences that we outline throughout the paper.

\section{Preliminaries and nomenclature}

In supervised learning, one is given a finite (training) dataset $\D = \set{(x_i,y_i)}_{i=1,\ldots,N}$ of $N$ samples, where  $x_i \in \X$ is the input data ({\em e.g.}, an image) and $y_i \in \Y$ is the output ({\em e.g.}, a discrete label). The goal is to learn (\ie, estimate the parameters $w$ of) a model $p_w$ (a parametric function) that maps inputs $x$ to estimated outputs $\hat y$, so that some loss (or risk) is minimized on unseen (test) data. It is common to assume that ${\cal D}$ consists of i.i.d.\ samples from some unknown distribution $p(x,y)$, which are used to compute a loss function that is minimized with respect to the parameters $w$ so that $p_w(y|x)$ is close to the ``true'' posterior $p(y|x)$.

In this work however, we make no assumption on the data generation process, nor do we attempt to approximate the true posterior $p(y|x)$, even if one existed. %
Instead, we adopt Kolmogorov's approach to learning the ``structure'' in the data, directly accounting for the finite sample size $N$, the approximation error in the loss $L_\D$, and the complexity of the model. 
The basic elements of Kolmogorov Complexity theory that we employ are described in \cite{vereshchagin2004kolmogorov} and summarized in
\Cref{sec:complexity-measures} to meake the paper self-contained. 

In Deep Learning, the optimization scheme acts as an implicit regularizer in the loss \cite{chaudhari2018stochastic}. For a measure of complexity to be relevant to Deep Learning, therefore, it must take into account both the loss function and the optimization scheme. To the best of our knowledge, none of the classical measures do so.

\subsection{Deep Neural Networks}
\label{sec:neural-networks}

Deep neural networks (DNNs) are a class of parametrized functions (models) obtained by successive composition of ``layers'' of affine operations, where  both  the linear terms and offset parameters are referred to as \emph{weights}, followed by element-wise non-linear operations such as a saturation or rectification.
One of the most common non-linearities is the rectified linear unit (ReLU), which leaves the positive part unchanged and zeroes the negative part. The first layer is thus of the form $x \mapsto h(Wx + b)$, where: $(W, b)$ are the weights, or improperly the ``weight vector''; $h(x)$ is defined as the component-wise maximum between $0$ and $x$. The output of the first layer is therefore given by $z = h(Wx + b)$.
The second layer performs an operation of the same type, with a different set of weights, taking as input the output of the first layer, and so on.
The last layer produces a normalized (probability) vector $p_w(y|x)$, with components $y \in \{1, \dots, K\}$, usually through a soft-max operation, {\em i.e.}, $z \mapsto e^z/\sum_k e^{z_k}$. The $k$-th entry of this vector represents the probability $p_w(y=k|x)$ of the input $x$ being of class $y=k$, as assessed by the model. The number of weights is typically larger than the number of samples in the training set.

The learning criterion is given by maximum likelihood: $\hat w = \arg\max_w L_\D(p_w)$, where $L_\D(p_w)$ is the {\em loss function} and is given by $L_\D(p_w) = \sum_{i=1}^N -\log p_w(y_i| x_i)$. This loss can be interpreted as an empirical approximation  of the average cross-entropy $H_{p, p_w}(y|x)$, which is zero when $p_w(y|x) = p(y|x)$. Note, however, that the latter posterior probability $p(y|x)$ is {\em not} among the infinitely many minimizers of the empirical loss,\footnote{The minimizers are of the form $\sum_{i=1}^N \delta(y_i - \phi(x_i))$ for any function $\phi$ such that $\phi(x_i) = y_i$. Such minimizers are said to {\em overfit} the data.} which therefore needs to be regularized. In Deep Learning, regularization is both explicit and implicit. Explicit mechanisms include additional terms in the loss function such as the norm of the weights (Weight Decay), or characteristics of the function class (architecture) such as downsampling and averaging (pooling). Implicit mechanisms include characteristics of the optimization.

The training process consists of minimizing the loss using one of the many variants of stochastic gradient descent (SGD). At every iteration, SGD makes a step in the direction of the (negative) gradient of $L_\D$, which is  approximated using a random subset of the training set (\emph{minibatch}). The length of the step is a hyperparameter called the {\em learning rate}.

A learning task can be {\em described} by the relation between the input variable $x$ ({\em e.g.}, an image) and the corresponding output variable $y$ ({\em e.g.}, `cat' or `dog'), and would be {\em defined} by the posterior distribution $p(y|x)$. Since that is typically not available, in practice the task is {\em specified} by a dataset $\D$ ({\em e.g.}, a collection of images with a label of either `cat' or `dog'): The dataset is all that is known about the desired goal before the training process starts. The {\em solution} to the task is determined by the dataset along with various {\em biases}, corresponding to modeling choices of loss function, class of functions, and optimization method.

\section{Complexity of a Learning Task}
\label{sec:task-complexity}

The same task could be specified by different datasets, with different numbers of samples and different size of the inputs. So, the complexity of the task is not just the complexity of the dataset: the popular MNIST  and CIFAR-10 classification tasks are similar on these counts, yet one is very ``simple'' to learn, the other is not. Instead, the complexity of a dataset is a function of two factors: The underlying \emph{structure} that describes the data collectively, as well as \emph{variability} that individual datapoints exhibit relative to the shared structure. This split is not unique: a dataset can be described or explained with different structures,  each leaving a different amount of residual variability unexplained. These two factors are captured in the following definition.
\begin{definition}
\label{def:task-complexity}
Let $\D = \set{(x_i, y_i)}_{i=1,\dotsc,N} \subseteq \X \times \Y$ be a dataset. We define the complexity of $\D$ as
\begin{equation}
\label{eq:complexity}
C(\D) = \min_{p(y \mid x)} L_\D(p) + K(p),
\end{equation}
where $L_\D(p) = \sum_{i=1}^N - \log p(y_i | x_i)$ is the empirical classification (cross-entropy) loss, and the minimum is over all possible computable probability distributions $p(y|x)$ of the label $y$ given the input $x$. By $K(p)$ we denote the Kolmogorov complexity of the distribution $p(y|x)$.
\end{definition}

\noindent At first sight, $C(\D)$ appears similar to a conditioned version of the two-part code in \cite[Appendix II]{vereshchagin2004kolmogorov}:
\begin{equation}
\label{eq:vereshchagin-complexity}
C_K(\D) = \min_{p(\y|\x)} -\log p(\y|\x) + K(p),
\end{equation}
where $\y = \langle y_1, \dotsc, y_N \rangle$ and $\x = \langle x_1, \dotsc, x_N \rangle$ are strings obtained by concatenating all labels and inputs of $\D$, respectively. In \Cref{prop:deterministic-complexity}.1, we  recall that $C_K(\D)$ coincides with the conditional Kolmogorov complexity $K(\y|\x)$ of the string $\y$ given $\x$. Instead, in \cref{eq:complexity} we only consider factorized distributions $p(y|x) = \prod_i p(y_i|x_i)$.
This has major consequences for determining the complexity of a task: The distribution $p(\y|\x)$ minimizing \cref{eq:vereshchagin-complexity} does not need to encode in $p$ all task-relevant information, as it can, at test time, extract any missing information from the training set $\x$ (cf.\ \Cref{prop:deterministic-complexity}.3).
Hence, $K(p)$ in \cref{eq:vereshchagin-complexity} alone would not be a valid measure of complexity. Instead, the distribution $p(y|x)$ minimizing \cref{eq:complexity} can only access a single datum $x$ at test time; hence, all structure in the dataset has to be encoded in $p$. In other words, the solution of \cref{eq:vereshchagin-complexity} does not need to capture the structure in ${\cal D}$, and therefore its complexity is not reflective of the complexity of the task. On the other hand, the solution of \cref{eq:complexity} is forced to capture the structure in the data.

Also note that, unlike $C_K(\D)$, $C(\D)$ is invariant to permutations of the dataset. Suppose that a dataset $\D$ for a binary classification task is ordered so that all negative samples $\D_-$ precede the positive samples $\D_+$.
Then we would have $C_K(\D) \leq \log |\D_-| + \log |\D_+|$ regardless of the complexity of the task, as it suffices to encode the number of negative and positive samples to reproduce the string $\y$ exactly.
However, we show in \Cref{prop:deterministic-complexity}.3 that making $p(\y|\x)$ permutation invariant does not yield a sensible measure of complexity of a task. To address permutation invariance, \cite{lidata} proposes the following definition, which uses deterministic functions rather than probability distributions:
\begin{equation}
C_\tdet(\D) = \min\set{K(f) \mid f \colon \X \to \Y \text{ such that } f(x_i)=y_i \;\;\forall\, (x_i,y_i) \in \D}.
\nonumber
\end{equation}
The following  proposition compares these definitions of complexity.

\begin{proposition}[Measures of complexity]
\label{prop:deterministic-complexity}
Up to an additive term which does not depend on the dataset $\D$, we have:
\begin{enumerate}[1.]
    \item $K(\y|\x) = C_K(\D)$.
    \item $C_K(\pi(\D)) \leq C(\D)$ for any permutation $\pi(\D)$ of $\D$.
    \item For every $C > 0$, there is a dataset $\D$ such that $C(\D) \geq C$ and $C_K(\pi(\D)) = O(1)$ for any permutation $\pi(\D)$.
    Therefore, $C(\D)$ is not simply the (average) complexity $C_K(\pi(\D))$ of encoding some permutation of $\D$.
    \item When $C_\tdet(\D)$ is defined, \ie, if there is a function $f \colon \X \to \Y$ such that $f(x_i)=y_i$ for each $(x_i,y_i) \in \D$, then we have $C(\D) \leq C_\tdet(\D)$. In addition, $C(\D) = C_\tdet(\D)$ if we have an oracle that provides a bijective map $\set{x_1, \ldots, x_N} \to \set{1, \ldots, N}$.
\end{enumerate}
\end{proposition}

\noindent Assuming the data are sampled i.i.d.\ from a computable probability distribution $p$, the following proposition characterizes the complexity of the dataset.
In particular, it shows that, asymptotically, the complexity of the dataset is given by the noise in the labels and the complexity of the distribution generating the data.
How to ignore the effect of noise on the complexity, and what happens in the non-asymptotic regime, is central to Kolmogorov's framework, which we will build upon in the next sections.

\begin{proposition}
\label{prop:iid-samples-complexity}
Fix a probability distribution $p(x,y)$ on $\X \times \Y$, and assume that $p(y|x)$ is computable.
If $\D$ is a collection of $N$ i.i.d.\ samples $(x_i, y_i) \sim p(x,y)$, then:
\begin{enumerate}[1.]
    \item The expected value of $C(\D)$ satisfies
\[ N \cdot H_p(y|x) \leq \E[C(\D)] \leq N\cdot H_p(y|x) + K(p(y|x)), \]
where $H_p(y|x)$ is the conditional entropy of $y$ given $x$.
    \item For any $\epsilon>0$ there is $N_0$ such that, with  probability $1-\epsilon$, for any $N \geq N_0$ we have the equality
    \[
    C(\D) = N\cdot H_p(y|x) + K(p(y|x)),
    \]
    and $p$ is the only computable distribution for which the equality holds.
\end{enumerate}
\end{proposition}

\noindent It is instructive to test our definition of complexity on a dataset of random labels, whereby each input is assigned a label at random. We will revisit this case often, as it challenges many of the extant theories of Deep Learning \cite{zhang2016understanding}.

\begin{example}[Complexity of random labels]
	\label{ex:random-label-complexity}
	Suppose that each input $x_i$ of the dataset $\D$ is associated to a label $y_i \in \Y$ sampled uniformly at random in a fixed finite set, so  $p(y|x) = 1 / |\Y|$ has a constant complexity.
	Under the same assumptions of \Cref{prop:deterministic-complexity}.4, the expected value of $C(\D)$ is $N \log |\Y| + O(1)$ by \Cref{prop:iid-samples-complexity}.
	Since $C(\D) \leq N \log|\Y| + O(1)$ for any such $\D$, the complexity of a ``typical'' dataset with random labels is approximately $N \log |\Y|$.
\end{example}

\noindent In a sense, learning random labels is a very complex task: They cannot be predicted from the input, so the model is forced to {\em memorize} them (``overfitting'').
Accordingly, $C(\D)$ in \Cref{ex:random-label-complexity} is very high. However, by construction there is \emph{no structure in the data}, which is why the model cannot generalize to unseen data. In other words, this is a complex memorization task, but a {\em trivial learning task}, as there is nothing to learn from the data. We would like the definition of complexity to reflect this distinction between learning and memorization. 

Another important aspect not captured by this definition of complexity is the role of performance in learning the structure in the data. For example, one can train a trivial model to distinguish images of airplanes from fireplaces by counting the number of blue pixels. It will not be very precise. To achieve a small error, however, one must learn what makes an airplane different from a fireplace even if the latter is painted blue.

\subsection{Structure Function of a Task}
\label{sec:structure-function}

The trade-off between the loss achievable by a solution $p$ on a dataset $\D$ and its complexity is captured by the \emph{Structure Function} \cite{vereshchagin2004kolmogorov}: 
\begin{equation}
\label{eq:structure-function}
S_{\D}(t) = \min_{K(p) \le t} L_\D(p). 
\end{equation}
It is a decreasing function that reaches zero for sufficiently high complexity, depending on the task.
As we increase complexity, the loss decreases rapidly while simple models correctly classify easy samples. 
After all shared structure is captured, the only way to further reduce the loss is to memorize the leftover samples, leading to the worst possible trade-off of one NAT of complexity for each unit of decrease in loss. Eventually, training enters this {\em linear (overfitting) regime for every dataset, given sufficient complexity}. For random labels, the entire training is in this overfitting regime since there is no structure in the data.

\begin{example}[Structure Function for random labels]
\label{ex:structure-function-random-labels}
By definition of $C(\D)$ we have $C(\D) \leq L_\D(p) + K(p)$ for any computable probability distribution $p(y|x)$.
For a typical dataset with random labels (see \Cref{ex:random-label-complexity}),
we have that $C(\D) \simeq N\log |\Y|$.
Therefore
\[
L_\D(p) \gtrsim N\log |\Y| - K(p),
\quad {\rm and \ hence} \quad
S_\D(t) \gtrsim N \log |\Y| - t.
\]
The lower bound for $S_\D(t)$ can be achieved by memorizing the label of $\lfloor t/\log |\Y|\rfloor$ data points.
Therefore $S_\D(t) \simeq N \log|\Y| - t$.
\end{example}

\noindent The Structure Function of a dataset cannot be computed in general (\cite[Section VII]{vereshchagin2004kolmogorov}). In \Cref{sec:information-in-the-weights} we introduce a generalized version that is computable: \Cref{fig:complexity-structure} shows the result on common datasets. The predicted fast decrease in the loss as complexity increases is followed by the asymptotic linear phase. The sharp transition to the linear regime is clearly visible as a function of the loss: As the parameter $\beta$ weighting complexity increases, a plateau is reached that depends on the task. Note that for random labels the loss decreases linearly as expected (left).

\subsection{Task Lagrangian and Minimal Sufficiency}
\label{sec:task-lagrangian}

The constrained optimization problem in the definition of the Structure Function (\cref{eq:structure-function}) has an associated Lagrangian $L_\D(p) + \beta K(p)$, where $\beta$ is a Lagrange multiplier that trades off the complexity of the model $K(p)$ with the fidelity $L_\D(p)$.
If we take the minimum over $p$, we obtain a family of complexity measures parametrized by $\beta$:
\begin{equation}
\label{eq:structure-lagrangian}
C_\beta(\D) = \min_{p(y|x)} L_\D(p) + \beta K(p).
\end{equation}
As a function of $\beta$, this is the Legendre transform of the Structure Function $S(t)$. To minimize $L_\D(p) + \beta K(p)$, we can increase the complexity $K(p)$ of the model until the return of doing so  has a ratio which is smaller than the constant $\beta$ we have selected.

If $p^*$ is a minimizer of \cref{eq:structure-lagrangian} for $\beta=1$, the corresponding Kolmogorov complexity $t^* = K(p^*)$ is the value at which the Structure Function $S(t)$ reaches a linear regime.
Thus, the special case $\beta = 1$ marks the transition to overfitting, and is related to Kolmogorov's notion of Minimal Sufficient Statistic \cite{vereshchagin2004kolmogorov}.
Since we are using Kolmogorov's complexity, $\beta=1$ is the worst possible trade-off. %

\begin{proposition}
\label{prop:worst-tradeoff}
  Given a task $\D$, let $\beta^*$ be the largest $\beta$ for which $C_\beta(\D)$ is not realized by a constant distribution $p(y|x)=p(y)$. Then $\beta^* \geq 1$, and $\beta^*=1$ if $\D$ is a typical dataset with random labels.
\end{proposition}

\noindent As we have seen for random labels, a dataset may be complex and yet exhibit little underlying structure.  We say that a distribution $p(y|x)$ is a Kolmogorov \emph{sufficient statistic} of $\D$ if it minimizes $C(\D)$. It is \emph{minimal} if it also minimizes $K(p)$ among all sufficient statistics, that is, the smallest statistic that is able to solves the task optimally. The rationale is that the smallest statistic that solves a task should not squander resources by memorizing nuisance variability. Rather, it should only capture the important information shared among the data. This is shown in the following example.

\begin{example}
For random labels, both the distribution $p(y|x)$ that memorizes all the labels in the dataset, and the uniform distribution $p(y|x)=1 / |\Y|$, are sufficient statistics. However, only the latter is minimal, since $K(p)$ is a constant which does not depend on $\D$. There is no structure to be extracted from a dataset of random labels.
\end{example}

\noindent The level of complexity of a model is an important design parameter the practitioner wishes to control. Rather than seeking minimal sufficient statistics, we explore the entire trade space, by introducing the notion of $\beta$-sufficiency.

\begin{definition}
\label{def:beta-sufficient-statistic}
Given a dataset $\D$, define a $\beta$-sufficient statistic of $\D$ as a probability distribution $p(y|x)$ such that $C_\beta(\D) = L_\D(p) + \beta K(p)$.
We say that $p(y|x)$ is a $\beta$-minimal sufficient statistic if it also minimizes $K(p)$ among all $\beta$-sufficient statistics.
\end{definition}

\noindent Notice that, for $\beta=1$, \Cref{def:beta-sufficient-statistic} reduces to minimal sufficiency in the sense of Kolmogorov.

\section{Asymmetric Distance between Tasks}
\label{sec:distance}

We now introduce a one-parameter family of distances between tasks. The hope is for them to correlate with the success of transfer learning. Since it is easier to learn a simple task from a complex one, it is desirable for the distance to be asymmetric.

\begin{definition}
    \label{def:kolmogorov-distance}
	The distance from task $\D_1$ to $\D_2$ at level $\beta$ is
	\[ d_\beta(\D_1 \to \D_2) := \max_{p_1} \min_{p_2} K(p_2|p_1) = \max_{p_1} \min_{p_2} K(\langle p_1,p_2\rangle) - K(p_1),
	\]
	where $p_i$ varies among all $\beta$-minimal sufficient statistics of $\D_i$.
\end{definition}

The intuition behind this definition is the following: for a task $\D_2$ to be close to $\D_1$, every $\beta$-minimal sufficient statistic $p_1$ of $\D_1$ should be close to some $\beta$-minimal sufficient statistic $p_2$ of $\D_2$.
Then, every optimal model $p_1$ of $\D_1$ can be fine-tuned to some optimal model $p_2$ of $\D_2$.

\begin{lem}
	\label{lemma:distance}
	The asymmetric distance $d_\beta$ satisfies the following properties:
	\begin{itemize}
		\item $d_\beta(\D_1 \to \D_2) \geq 0$ \quad (positivity);
		\item $d_\beta(\D \to \D) = O(1)$ \quad (a task is close to itself);
		\item $d_\beta(\D_1 \to \D_3) \leq d_\beta(\D_1 \to \D_2) + d_\beta(\D_2 \to \D_3) + O(1)$ \quad (triangle inequality).
	\end{itemize}
\end{lem}

\noindent We now derive a characterization of the distance between tasks based on the complexity of their composition, amenable to generalization in \Cref{sec:generalized-distance}.  Denote by $\D_1 \sqcup \D_2$ the disjoint union of two datasets $\D_1$ and $\D_2$, defined as
\[ \D_1 \sqcup \D_2 = \{ (\langle i, x \rangle, y) \mid i \in \{1,2\}, \; (x,y) \in \D_i \}. \]
Notice that an index $i$ is added to the input, in order to recognize the original dataset.
A desirable property for a distance between tasks would be that $d_\beta(\D_1 \sqcup \D_2 \to \D_1) = O(1)$: Indeed, a model that performs well on $\D_1 \sqcup \D_2$ should be easily fine-tuned to a model that performs well on $\D_1$ alone.
Adding the index $i$ in the definition of $\D_1 \sqcup \D_2$ is essential, as we can see in the following example.

\begin{example}
Let $\D$ be a typical dataset with random labels, and let $\D_1 \subseteq \D$ be the set of data points $(x,y) \in \D$ satisfying some property of Kolmogorov complexity $t \gg 0$.
Then the whole dataset $\D$ has a trivial structure, whereas $\D_1$ has a complicated structure.
If both $|\D_1|$ and $|\D \setminus \D_1|$ are large, by \Cref{prop:iid-samples-complexity} the Kolmogorov complexity of a minimal sufficient statistic of $\D_1$ is $t$, and the complexity of a minimal sufficient statistic of $\D$ is $O(1)$.
\end{example}

\noindent We now prove that, under the hypotheses of \Cref{prop:iid-samples-complexity}.2, the property $d_\beta(\D_1 \sqcup \D_2 \to \D_1) = O(1)$ is satisfied.

\begin{thm}
	\label{prop:concatenation}
	Suppose that $\D_1$ and $\D_2$ are obtained by sampling from two fixed distributions $p_1$ and $p_2$ on $\X \times \Y$, such that $p_1(y|x)$ and $p_2(y|x)$ are computable.
	Then, with high probability and for $|\D_1|$ and $|\D_2|$ sufficiently large, $d_\beta(\D_1 \sqcup \D_2 \to \D_1) = O(1)$.
\end{thm}

\begin{cor}
  \label{cor:union-distance}
	Under the hypothesis of \Cref{prop:concatenation}, we have that
	\[ d_\beta(\D_1 \to \D_2) = d_\beta(\D_1 \to \D_1 \sqcup \D_2) + O(1) = \max_{p_{1}}\min_{p_{12}} K(p_{12}) - K(p_1) + O(1), \]
	where $p_{12}$ varies among the $\beta$-minimal sufficient statistics of $\D_{12}$ and $p_{12}$ varies among those of $\D_1$.
\end{cor}

We now have a way of comparing different learning tasks, at least in theory. The asymmetric distance $d_\beta(\D_1 \to \D_2)$ allows us to quantify how difficult it is to learn a new task $\D_2$ given a solution to $\D_1$.
However, quantities defined in terms of Kolmogorov complexity are difficult to handle in practice, and may behave well only in an asymptotic regime. In the next section, we introduce a generalization of the framework developed so far that can be instantiated for a particular model class such as deep neural networks.

\section{Information in the Model Parameters}
\label{sec:information-in-the-weights}

Whatever structure or ``information'' was captured from the dataset, it ought to be measurable from the model parameters, since they are all we have left after training. As we will see in the next section, this intuition is faulty, as {\em how} we converge to a set of parameter (\ie, the optimization algorithm) also affects what information we can extract from the data.  For now, we focus on generalizing the theory in the previous section with an eye towards computability. Although most of our arguments are general, we focus on deep neural networks (DNNs) as a model class. They can have millions of parameters, so measuring their information can be non-trivial.

One way to compute information in the parameters is to measure their coding length at some level of precision, independent of the particular task.
This is suboptimal, as only a small subset of the weights of a trained neural networks matters: Imagine changing a certain component of the weights, and observing no change in the loss. Arguably, that weight ``contains no information'' about the dataset.
For the purpose of storing the trained model, that weight could be replaced with any constant, or randomized each time the network is used.
The loss landscape has a small curvature\footnote{We will elaborate on this point after \Cref{prop:fisher-complexity-equivalence}.} in the coordinate direction corresponding to that weight. 
On the other hand, imagine changing the least significant bit of another component of the weights and noticing a large increase in the loss.
That weight could be said to be very ``informative,''  and its valued should be stored with high precision.

With these observations in mind, we allow the weights to be encoded with some uncertainty, through a probability distribution $Q(w|\D)$ which depends on the dataset $\D$.
For example, Dirac's Delta $Q(w|\D) = \delta_{w^*}$ corresponds to an exact encoding of the weight vector $w^*$. If we fix a reference ``pre'' distribution $P(w)$, \cite{hinton1993keeping} shows that the labels $\y$ can be reconstructed from the input $\x$ and the pre-distribution $P(w)$, by using
\[ \E_{w \sim Q(w|\D)} [L_\D(p_w(y|x))] + \KL{Q(w|\D)}{P(w)} \] additional NATS.
This expression resembles the right-hand side of \cref{eq:complexity} in that it describes a trade-off between fidelity and complexity. Here, Kolmogorov complexity has been replaced by the Kullbach-Liebler (KL) divergence $\KL{Q(w|\D)}{P(w)}$ which we call the \emph{information in the parameters} of the model.%
\footnote{
This should not be confused with the information in $\x$ and $\y$, or any intermediate representation that is build by the model, such as the activations of a DNN, which is more frequently studied. Information in the parameters and information in the activations are different and, in the case of DNNs, are related through the Emergence Bound \cite{achille2017emergence}.
}
This leads to the following new definition of complexity. It entails two arbitrary choices of distributions: One that is chosen before any data is observed, which we call ``pre-distribution'' $P(w)$, and one that is chosen at the end of training, which we call ``post-distribution'' $Q(w | {\cal D})$. We deliberately avoid calling them ``prior'' and ``posterior'' to avoid any confusion with a Bayesian setting. Here, the weights $w$ are fixed at the end of training, so if the post-distribution was the true posterior, it would be degenerate. One should also not confuse the post-distribution with the distribution of the weights {\em during training}, due for instance to the stochasticity of the optimization algorithm, on which we will comment later.

\begin{definition}
\label{def:complexity-dnn}
	The complexity of the task $\D$ at level $\beta$, using the post-distribution $Q(w|\D)$ and the pre-distribution $P(w)$, is given by
	\begin{equation}
  \label{eq:kl-complexity}
	C_\beta(\D; P, Q) = \E_{w \sim Q(w|\D)}[L_\D(p_w(y|x))] + \beta \!\!\! \underbrace{\KL{Q(w|\D)}{P(w)}}_{\rm information \ in \ the \ parameters} \!\!\!.
  \end{equation}
	The second term, $\KL{Q(w|\D)}{P(w)}$, measures the information in the parameters of the model. We refer to $\E_{w \sim Q(w|\D)}[L_\D(p_w(y|x))]$ as the (expected) reconstruction error of the label under the hypothesis $Q(w|\D)$.
\end{definition}

\noindent We emphasize again that there is no implied Bayesian interpretation,  as $Q(w|\D)$ can be \emph{any} distribution. Depending on the choice, this expression can be computed in closed form or estimated (\Cref{sec:complexity-measures}). 
For instance, when  $Q(w|\D) = \delta_{w^*}$, the expression reduces to the length of a two-part code for $\D$ using the model class $p_w(y|x)$.
However, \Cref{def:complexity-dnn} is more general and can be extended to the continuous case, or in cases where there is a {\em bona fide} distribution, as in variational inference and Bayesian Neural Networks.

Another fundamental difference of our framework compared to prior work in the area is that, while \cref{eq:complexity} measures the complexity in terms of the best obtainable by the model class (in that case, the class of computable probability distributions), the complexity $C_\beta(\D; P, Q)$ takes into account both the particular model class and the \emph{training algorithm}, \ie, the map $A: \D \mapsto Q(w|\D)$, as we shall see in \Cref{sec:task-accessibility}.

\subsection{Relation with Kolmogorov, Shannon, and Fisher Information}
\label{sec:complexity-measures}

Of all the possible choices of pre-distribution $P(w)$ in \Cref{def:complexity-dnn}, we investigate three special cases: The ``universal prior'' of all computable distributions; an ``adapted prior'' which relies on a probability distribution over datasets; an uninformative prior, agnostic of the dataset.

We start with the first case, which provides a link between \Cref{def:complexity-dnn} and the framework of \Cref{sec:task-complexity}.
For a given weight vector $w$, we define the universal prior $P(w) = \frac{1}{Z} e^{-K(w)}$,
where $Z$ is a normalization constant.
This can be interpreted as follows: for every $w$, choose a minimal program that outputs $w$, and assign it a probability which decreases exponentially in terms of the length of the program.

\begin{proposition}[Kolmogorov Complexity of the Weights]
  \label{prop:kolmogorov-complexity-equivalence}
	Let $P(w)$ be the universal prior, and let $Q(w|\D) = \delta_{w^*}$ be a Dirac delta.
	Then the information in the weights equals the Kolmogorov complexity of the weights $K(w^*)$, up to a constant.
\end{proposition}

We now turn to the second case, which provides a link with Shannon mutual information.

\begin{proposition}[Shannon Information in the Weights]
  \label{prop:shannon-complexity-equivalence}
	Assume the dataset $\D$ is sampled from a distribution $\pi(\D)$, and let the outcome of training on a sampled dataset $\D$ be described by a distribution $Q(w | \D)$. Then the pre-distribution $P(w)$ minimizing the expected complexity $\E_\D[C_\beta(\D; P, Q)]$ is $P(w) = \E_\D[Q(w|\D)]$, and the expected information in the weights is given by
	\begin{equation}
		\E_\D[\KL{Q(w|\D)}{P(w)}] = I(w; \D).
		\nonumber
	\end{equation}
	Here $I(w; \D)$ is Shannon's mutual information between the weights and the dataset, where the weights are seen as a (stochastic) function of the dataset given by the training algorithm (\eg, SGD).
\end{proposition}

Note that, in this case, the pre-distribution $P(w)$ is optimal given the choice of the training algorithm (i.e., the map $A: \D \to Q(w|\D)$) and the distribution of training datasets $\pi(\D)$.
However, the distribution $\pi(\D)$ is generally unknown, as we are often given a single dataset to train.
Even if it was known, computing the marginal distribution $\E_\D[Q(w|\D)]$ over all possible datasets would not be realistic, as it is high-dimensional and has complex interactions between different components. Nevertheless, it is interesting that the information in the parameters specializes to Shannon's mutual information in the weights \cite{achille2017emergence}.

The third case, namely an uninformative prior with a Gaussian posterior, is the most practical, and provides a link to the Fisher Information Matrix and the learning dynamics of common optimization algorithms such as SGD.

\begin{thm}[Fisher Information in the Weights]
  \label{prop:fisher-complexity-equivalence}
	Choose an isotropic Gaussian pre-distribution $P(w) \sim N(0, \lambda^2 I)$.
	Let the post-distribution $Q(w|\D)$ also be Gaussian: $Q(w|\D) \sim N(w^*, \Sigma)$, where $w^*$ is a local minimum of the cross-entropy loss.
	Then, for $\lambda \to \infty$, we have that:
	\begin{itemize}
		\item the covariance $\Sigma^*$ which minimizes $C_\beta(\D;P,Q)$ tends to
		$ \beta H^{-1} = \frac{\beta}{N}F^{-1}$ (this is in accordance with the Cram\'er-Rao bound);

		\item the information in the weights is given by
		\[ \KL{Q(w|\D)}{P(w)} = \frac{1}{2} \log |F| + \frac{1}{2} k \log \lambda^2 + O(1). \qedhere \]
	\end{itemize}
\end{thm}

\noindent Recalling that the Fisher Information $F$ measures the local curvature, this proposition confirms the qualitative discussion at the beginning of \Cref{sec:information-in-the-weights}: The optimal covariance $\Sigma^* \propto H^{-1}$ gives high variability to the directions of low curvature, which are ``less informative,'' whereas it gives low variability to the ``more informative'' directions of high curvature.
The Fisher Information describes the information  contained in the weights about the dataset. In \Cref{sec:generalized-distance} we discuss how to compute it.

\subsection{Connections with the PAC-Bayes Bound}
\label{sec:pac-bayes}

The Lagrangian $C_\beta(\D; P, Q)$ admits another interpretation as an upper-bound to the test error, as shown by the PAC-Bayes test error bound:

\begin{thm}[{\cite[Theorem 2]{mcallester2013pac}}]
Assume the dataset $\D=\{(x_i, y_i)\}_{i=1}^N$ is sampled i.i.d.\ from a distribution $p(y,x)$, and assume that the per-sample loss used to train is bounded by $L_\text{max} = 1$ (we can reduce to this case by clipping and rescaling the loss).
For any fixed $\beta>1/2$,  pre-distribution $P(w)$, and  post-distribution $Q(w|\D)$, with probability at least $1-\delta$ over the sample of $\D$, we have:
\begin{align}
\label{eq:pac-bayes-bound}
L_{\text{test}}(Q) &\leq \frac{1}{N(1-\frac{1}{2\beta})} \Big[\E_{w \sim Q(w|\D)}[L_\D(p_w)] + \beta \KL{Q}{P} + \beta \log \frac{1}{\delta}  \Big] \nonumber \\
&= \frac{1}{N(1-\frac{1}{2\beta})} \Big[C_{\beta}(\D; P, Q) + \beta  \log \frac{1}{\delta}\Big]. \nonumber
\end{align}
where  $L_\text{test}(Q) := \E_{x,y \sim p(x,y)} [\E_{w\sim Q}[p_w(y|x)]]$ is the expected per-sample test error that the model incurs using the weight distribution $Q(w|\D)$.
\end{thm}
\noindent Hence, we see that minimizing the Lagrangian $C_\beta(\D; P, Q)$ can be interpreted as minimizing an upper-bound on the test error of the model, rather than directly minimizing the train error.
This is in accordance with the intuition developed earlier, that minimizing $C_\beta(\D; P, Q)$ forces the model to capture the  structure of the data. It is also interesting to consider the following bound on the expectation over the sampling of $\D$ (\cite[Theorem 4]{mcallester2013pac}):
\[
\E_\D [L_{\text{test}}(Q(w|\D))] \leq \frac{1}{N(1-\frac{1}{2\beta})} \Big[\E_\D [L_\D(Q(w|\D))] + \beta \, \E_\D[ \KL{Q(w|\D)}{P}] \Big].
\]
As we have seen in \Cref{prop:shannon-complexity-equivalence}, for the optimal choice of pre-distribution $P$ minimizing the bound, we have $\E_\D [\KL{Q}{P}] = I(w; \D)$.
Hence, the Shannon Information that the weights of the model have about the dataset $\D$ is the measure of complexity that gives (on expectation) the strongest generalization bound.
This has also been noted in \cite{achille2017emergence}.
In \cite{dziugaite2017computing}, a non-vacuous generalization bound is computed for DNNs, using (non-centered and non-isotropic) Gaussian prior and posterior distributions.

\section{Generalized distance, reachability, and learnability of tasks}
\label{sec:generalized-distance}

Unlike $C_\beta(\D)$, the definition of  $C_\beta(\D;P,Q)$ in \cref{eq:kl-complexity} captures the complexity of a dataset for a particular model class and training algorithm. Motivated by this, we now  define a distance between datasets which is tailored to the model.

Throughout this section, fix a parametrized model class $p_w(y|x)$, a pre-distribution $P(w)$, and a class $\mathcal{Q}$ of post-distributions $Q(w|\D)$.
The case we are most interested in is that of DNNs, with an uninformative pre- and a Gaussian post-distributionr (see \Cref{sec:complexity-measures}).
Our starting point is to generalize Kolmogorov's Structure Function.
Consider the following generalized Structure Function:
\[
S_\D(t) = \, \min_{\KL{Q(w|\D)}{P(w)} \leq t} \, \E_{w\sim Q(w|\D)}[L_\D(p_w)].
\]
Here the minimum is taken among all post-distributions $Q(w|\D)$ in the chosen class $\mathcal{Q}$.
Similarly to what we have seen in \Cref{sec:task-lagrangian}, this minimization problem has $C_\beta(\D;P,Q)$ as its associated Lagrangian.
We say that $Q(w|\D)$ is a $\beta$-sufficient statistic if it minimizes $C_\beta(\D;P,Q)$. It is a $\beta$-minimal sufficient statistic if it also minimizes $\KL{Q(w|\D)}{P(w)}$.
Motivated by \Cref{cor:union-distance}, we then introduce the following distance.

\begin{definition}
\label{def:kl-distance}
The task distance from $\D_1$ to $\D_2$ at level $\beta$ is
\[
  d_\beta(\D_1 \to \D_2) = \max_{Q_{1}} \min_{Q_{12}} \KL{Q_{12}(w|\D)}{P(w)} - \KL{Q_1(w|\D)}{P(w)},
\]
where $Q_{12}(w|\D)$ is a $\beta$-minimal sufficient statistic for $\D_1 \sqcup \D_2$, and $Q_1(w|\D)$ is a $\beta$-minimal sufficient statistic for $\D_1$.
\end{definition}

While more general and amenable to computation than the distance of \Cref{def:kolmogorov-distance}, this is not an actual distance as it lacks several of the properties in \Cref{lemma:distance}.
Nonetheless, we will show that $d_\beta$ does indeed capture the difficulty of fine-tuning from one dataset to another using specific model families (such as DNNs) and training algorithms (SGD), and empirically shows good correlation with other distances (\eg, taxonomical), when those are defined.

\subsection{Local Task Reachability}

Until now, we have only considered global minimization problems, where we aim to find the best solution satisfying some optimal trade-off. However, many learning algorithm (\eg, SGD) are local: Starting from the current solution, they take a small step to greedily minimize some objective function. This raises the question of which conditions allow such an algorithm to recover an optimal solution.

Given a distribution $Q(w|\D) \in \mathcal{Q}$, denote by $L_\D(Q)$ the expected loss $\E_{w \sim Q(w|\D)}[L_\D(p_w)]$.
Fix a distance $d$ on $\mathcal{Q}$ such that both $L_\D(Q)$ and $\KL{Q}{P}$ are continuous, as functions $\mathcal{Q} \to \R$.
This way, the Lagrangian $C_\beta(\D; P, Q) = L_\D(Q) + \beta \KL{Q}{P}$ is continuous in the joint variable $(Q, \beta)$.
In the case of DNNs, where $P$ is uninformative and $\mathcal{Q}$ is the class of the Gaussian distributions $Q(w|\D) \sim N(w^*, \Sigma)$, we can take for example $d$ as the Wasserstein distance, or the Euclidean distance between the parameters $(w^*, \Sigma)$.

\begin{definition}[$\epsilon$-local learning algorithm]
  \label{def:local-learning-algorithm}
  Fix $\beta \geq 0$.
  We say that a step is $\epsilon$-\emph{local} if, starting from a given statistic $Q_0$, it finds the statistic $Q$ that minimizes
  \[
  C_\beta(\D; P,Q) = L_\D(Q) + \beta \KL{Q}{P}
  \]
  and such that $d(Q, Q_0) \leq \epsilon$.
  We say that a learning algorithm is $\epsilon$-\emph{local} if it only takes $\epsilon$-local steps.
\end{definition}

In the limit $\epsilon \to 0$, this reduces to gradient descent on the Lagrangian $C_\beta$.
Notice however that this is not the same as performing gradient descent on the cross-entropy loss $L_\D$, unless $\beta=0$.
Indeed, minimizing $L_\D$, unlike minimizing $C_\beta$ (\Cref{sec:pac-bayes}), gives no guarantees on the performance on test data, as the learning algorithm could simply memorize the dataset $\D$.

We will show in the next section than a DNN trained with SGD can actually be interpreted as a local learning algorithm minimizing $C_\beta$. Common methods of training DNNs also rely on annealing of the learning rate: training starts with a high learning rate which is anneals gradually.
As we will make precise in the next section, this corresponds to starting with a high value of $\beta$, and gradually decreasing it to a final value $\bar\beta$.
This helps avoiding degenerate solutions, because the model starts by favouring structure over memorization. We introduce the following definition to capture the role of annealing.

\begin{definition}[$\epsilon$-local learning algorithm with annealing]
An $\epsilon$-local learning algorithm with annealing is a learning algorithm that alternates $\epsilon$-local steps (that change the distribution $Q$) and annealing steps (that can decrease the value of $\beta$).
\end{definition}

Notice that, if the annealing is slow, then an $\epsilon$-local learning algorithm with annealing can be regarded as a discrete gradient descent with respect to the joint variable $(Q, \beta)$.
A natural question is: Does an $\epsilon$-local learning algorithm always recover a global minimum? The following result gives a sufficient condition for an $\epsilon$-local learning algorithm with annealing to recover the global minimum of $C_{\bar\beta}$. 

\begin{prop}
\label{prop:local-learning-algorithm}
Fix an annealing schedule $\beta_0 \geq \beta_1 \geq \dotsb \geq \beta_n = \bar\beta$.
Suppose that, for every global minimizer $Q_i$ of $C_{\beta_i}$, there exists a global minimizer $Q_{i+1}$ of $C|_{\beta_{i+1}}$ with $d(Q_i, Q_{i+1}) \leq \epsilon$.
Then, an $\epsilon$-local learning algorithm with annealing that starts from a global minimizer $Q_0$ of $C|_{\beta_0}$, and performs one $\epsilon$-local step after each annealing step, computes a global minimizer of $C_{\bar\beta}$.
\end{prop}

We say that a task that satisfies the conditions of \Cref{prop:local-learning-algorithm} for some annealing schedule is \emph{$\epsilon$-connected}.
However, in general we cannot guarantee that an $\epsilon$-local learning algorithm  cannot get  stuck in a local minimum of $C_{\bar\beta}$.
For example, this is bound to happen if there is no sequence $(Q_0, \beta_0), (Q_1, \beta_1), \dotsc, (Q_n, \beta_n) = (\bar Q, \bar\beta)$ such that $\bar Q$ is a global minimum of $C_{\bar\beta}$, $d(Q_i, Q_{i+1}) \leq \epsilon$, and $C_{\beta_i}(\D; P, Q_i) \leq C_{\beta_{i+1}}(\D; P, Q_{i+1})$ for all $i$. In particular, this happens if there is no continuous path of global minima, which is an intrinsic property of the task $\D$ and loss function with respect to the function class.

\subsection{SGD as a local learning algorithm}
\label{sec:task-accessibility}

So far, we have introduced an abstract notion of distance between tasks.
However, we have not yet shown that this notion is useful for DNNs, or that indeed SGD is a local algorithm in the sense of \Cref{def:local-learning-algorithm}.

In \cite{achille2018dynamics}, a step in this direction is taken.
It is shown that, to first approximation, the probability of SGD converging to a configuration $w_f$ solving task $\D$ in a given time $t_f - t_0$, starting from a configuration $w_0$, is given by:
\begin{equation}
\label{eq:transition-probability}
p(w_f,t_f|w_0, t_0) \simeq \utext{\vphantom{\int_{w_0}^{w_f}}e^{-\frac{1}{2T} \Delta (L_{\D}(w) + \beta \KL{Q}{P})}}{static part} \ \utext{\int_{w_0}^{w_f}  e^{-\frac{1}{2D}\int_{t_0}^{t_f} \frac{1}{2} \dot{u}(t)^2 + L_\D(u(t)) \d t} \d u(t)}{dynamic part},
\end{equation}
when using the pre-distribution $P(w) \sim N(0, \lambda^2 I)$, the optimal Gaussian post-distribution $Q(w|\D) \sim N(w^*, \Sigma^*)$, and $\beta = 2 \lambda^2 \gamma T$.
Here $\Delta f(w) := f(w_f) - f(w_0)$,  $\gamma$ is the weight decay coefficient used to train the network, and $T \propto \eta/B$ is a temperature parameter that depends on the learning rate $\eta$ and the batch size $B$.

From \cref{eq:transition-probability} we see that, with high probability, SGD takes steps that minimize the effective potential $C_\beta(\D; P, Q) = L_{\D}(w) + \beta \KL{Q}{P})$ (static part), while trying to minimize the distance traveled in the given time (dynamic part). Hence, SGD can be seen as a stochastic implementation of a local learning algorithm in the sense of \Cref{def:local-learning-algorithm}.

In particular, this has the non-intuitive implication that SGD does not greedily optimize the loss function as it may seem from its update equation.
Rather, on average, the complexity of the recovered solution affects the dynamics of the optimization process, and changes the objective.\footnote{This was also observed by \cite{chaudhari2018stochastic} although derived for the continuous approximation of SGD as a stochastic differential equation.} 
Therefore, the complexity we have introduced in \Cref{sec:information-in-the-weights} is not simply an abstract means to study different trade-offs, but rather plays a concrete role in the learning dynamics.

Since SGD can be seen as a local learning algorithm, annealing the learning rate during training (\ie, annealing the parameter $\beta \propto T$) can be interpreted as a way of learning the structure of a task by slowly sweeping the Structure Function.
Hence, SGD with annealing adaptively changes the complexity of the model, even if its dimensionality is fixed at the outset. This creates non-trivial dynamics that turn out to be beneficial to avoid memorization. It also points to the importance of the initial transient of learning during the annealing, a prime area for investigation beyond the asymptotics of the particular value of the weights, or the local structure of the residual around them, at convergence. Another consequence is that, when the task is not $\epsilon$-connected, SGD may fail to recover the optimal structure of the data.

One may wonder if there are examples of simple tasks which are severely non-$\epsilon$-connected, and if SGD actually fails to solve them. Biology provides such an example. \cite{achille2017critical} show that there exists tasks -- a dataset of images, and the same dataset with the same images slightly blurred  -- which are close to each other in any intuitive sense (the taxonomical distance is zero), yet a network trained on the one cannot be fine-tuned to solve the other.
This phenomenon, observed in diverse tasks and networks, across biological and artificial systems, is shown to be related to changes undergone by the Fisher Information Matrix during training. Within our framework, this can be interpreted as biasing the initial optimization process toward a  minimum of $C_\beta(\D', P, Q)$ for the blurred task $\D'$ which is not close to any of the global minima, despite  being a local minimum of the Lagrangian $C_\beta(\D, P, Q)$. Hence, the local learning algorithm, rather than learning the structure of the data, performs sub-optimal greedy choices to optimize the loss.

Even though the static term of \cref{eq:transition-probability} is high (\ie, the distance is small), this example indicates that that the dynamic term is not negligible. This suggests additional investigations are needed into the  dynamics of differential learning.

\section{Empirical validation}
\label{sec:experiments}

\begin{figure}
	\centering
    \includegraphics[width=.30\linewidth]{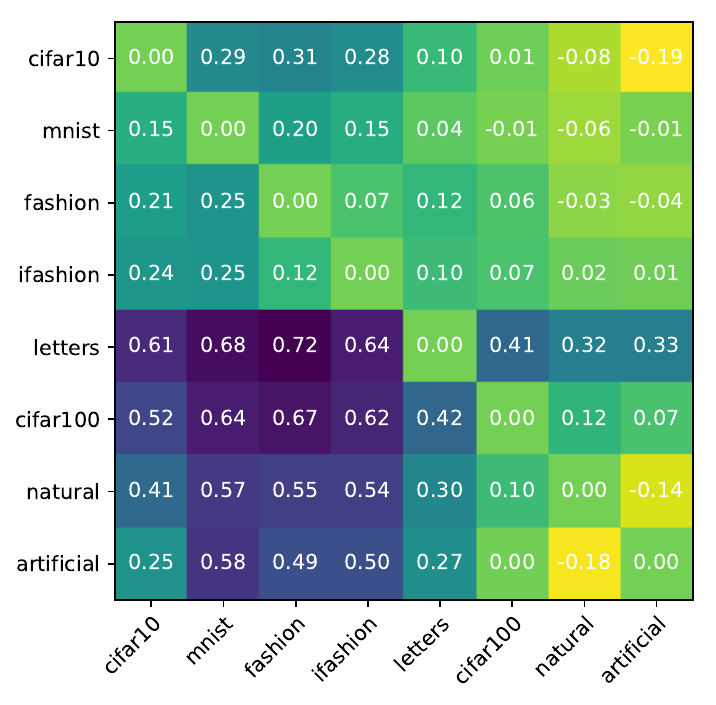}
    \hspace{1.5cm}
    \includegraphics[width=.35\linewidth]{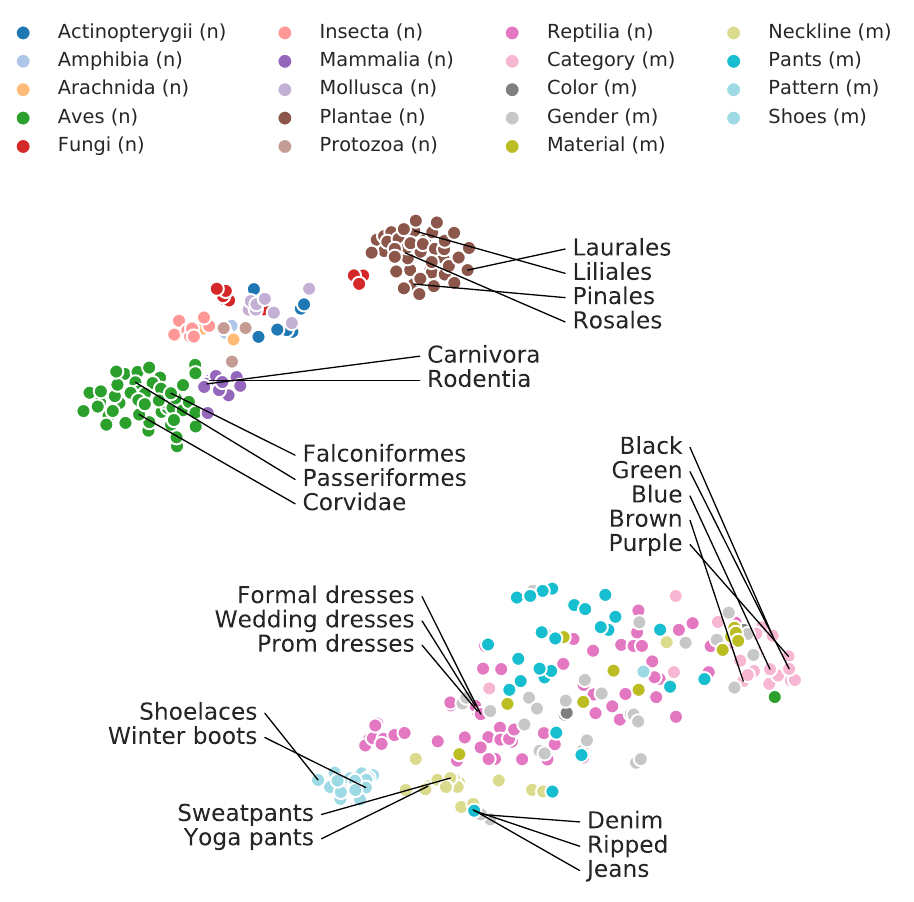}
	\caption{
	\textbf{(Left)} \textbf{Estimated distance matrix between several tasks.}
	Each entry shows the distance $d_\beta(\D_1 \to \D_2)$ going from the column task $\D_1$ to the row task $\D_2$.
    Going from a complex task like CIFAR-100 to a simpler task (like MNIST) is always easier than the converse.
	Subtasks are close to the full tasks (\eg, the subset of ``artificial'' and ``natural'' objects of CIFAR-100 are both close to CIFAR-100).
	Similar tasks on the domain of small black and white images (Fashion MNIST, MNIST, Letters) are also closer together than to natural images. Inverting the colors on Fashion  images leads to a very similar task (I-Fashion), as expected.
	\textbf{(Right)} \textbf{T-SNE embedding} of several organism species classification tasks and clothing attributes classification tasks based on their distance, reproduced from \cite{achille2019task2vec}, which uses a similar definition of task distance based on the Fisher Information. Intuitively, similar tasks cluster together. In the case of species classification, this largely follows the taxonomical structure.
	}
\label{fig:distance}
\end{figure}

The theoretical framework developed in this paper has tangible ramifications.
A robust version of the asymmetric distance between tasks described here has been used in \cite{achille2019task2vec}, to create a metric embedding of hundreds of real-world datasets (\Cref{fig:distance}, right).
The structure of the embedding shows good accordance with the complexity of the tasks (\Cref{fig:complexity-structure}), and both with intuitive notions of similarity and with other metrics, such as taxonomical distance, when those are available (\Cref{fig:distance}).
The metric embedding allows tackling important meta-tasks, such as predicting the performance of a pre-trained model on a new datasets, given its performance on past similar datasets.
This greatly simplifies the process of recommending a pre-trained model as the starting point for fine-tuning a new task, which otherwise would require running a large experiment to determine the best model.
The results are also shown to improve compared to pre-training a model on the most complex available task, \ie, full ImageNet classification.

\begin{figure}
    \centering
    \includegraphics[width=.8\linewidth]{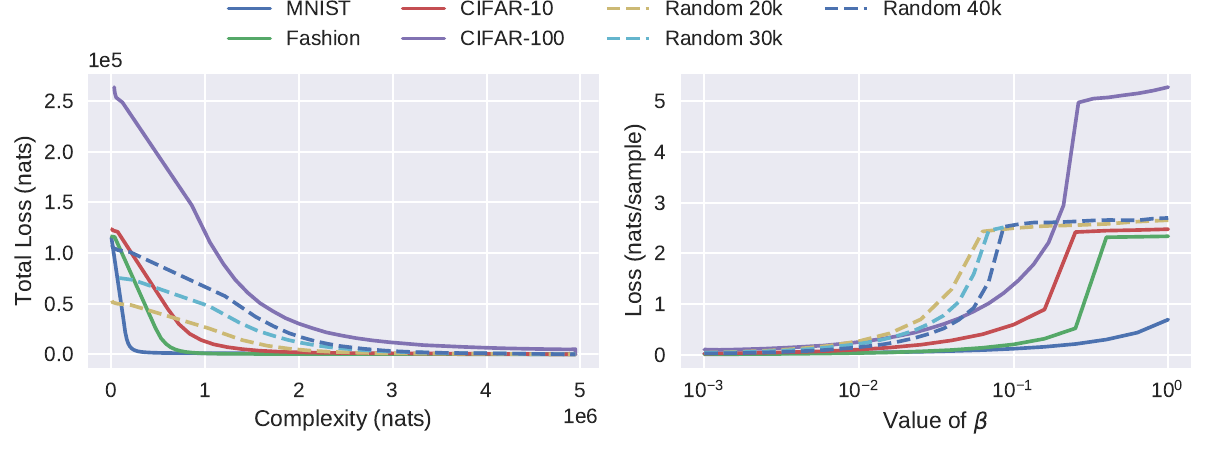}
    \caption{
    We train an AllCNN model \cite{springenberg2014striving} by minimizing the lagrangian $C_\beta(\D;P,Q)$ of \cref{eq:kl-complexity}, for different values of $\beta$.
    We carry out this experiment on four real datasets (MNIST, Fashion-MNIST, CIFAR-10, CIFAR-100) and three artificial datasets of different size ($20$k, $30$k, $40$k samples) where labels are assigned randomly.
    In all experiments, we first pretrain for $180$ epochs with weight decay $5\cdot 10^{-4}$, batch size $128$, and learning rate $0.05$ (reduced by a factor of $10$ after $80$ and $120$ epochs), by minimizing the cross-entropy loss $L_\D$.
    Then we fine-tune for $5$ epochs, with a learning rate of $0.001$ for the batch normalization parameters and for the parameters of the final classifier, and of $0.1$ for all other parameters.
    In each epoch of fine-tuning, only $10$k random samples from the training set are used.
    \textbf{(Left)} Approximation of the Structure Function $S_\D(t)$, obtained by plotting the total loss $C_\beta(\D;P,Q)$ against the information in the model parameters.
    Notice that the loss on simpler datasets rapidly decreases as we increase the complexity of the model. On the other hand, for random labels, the decrease in error is much slower and follows an almost linear trend, as described in \Cref{ex:structure-function-random-labels}.
    \textbf{(Right)} Plot of the cross-entropy loss $L_\D$ as a function of $\beta$ (which controls the trade-off between complexity and error).
    All datasets show a sharp transition between overfitting and underfitting at some value of $\beta$.
    Simpler datasets can still be fitted for a high $\beta$, since models of low complexity can already correctly classify the data.
    On the other hand, more complex datasets need a very low $\beta$ to fit the data: In particular, random labels have the worst trade-off.
    The transition happens at a value that depends on the complexity of the data distribution, and is mostly independent of the dataset size: the random labeled datasets all transition at a similar point, despite the difference in size.
    }
    \label{fig:complexity-structure}
\end{figure}

\begin{figure}
    \centering
    \includegraphics[width=.8\linewidth]{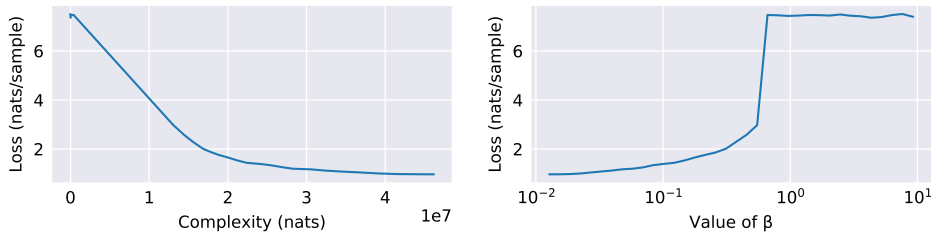}
    \caption{The same experiment as in \Cref{fig:complexity-structure}, on the ImageNet dataset \cite{deng2009imagenet} and using a pretrained ResNet-34 \cite{he2016deep}.}
    \label{fig:complexity-structure-imagenet}
\end{figure}

Besides these independent validations, we also report some additional experiments that illustrate the concepts introduced in this paper.
In \Cref{fig:complexity-structure} (right), we plot the loss as a function of $\beta$.
As predicted by \Cref{prop:worst-tradeoff}, the Lagrangian $C_\beta(\D; P,Q)$ exhibits sharp ``phase transitions'' at a critical value $\beta^*$, when the model transitions from fitting the dataset with a trivial uniform distribution to actually fitting the labels. Regardless of their size, datasets of random labels transition around the same value $\beta^*$.
The other datasets transition at a higher value, which depends on their complexity.
For example, a complex dataset such a CIFAR-100 transitions at a much lower $\beta^*$ than a simple dataset such as MNIST.
Notice that, in this experiment, the critical value for random labels is not $\beta^*=1$: This is because the complexity is computed using an uninformative prior (see \Cref{sec:complexity-measures}), and not the universal prior as in \Cref{prop:worst-tradeoff}.
In \Cref{fig:complexity-structure} (left), as the complexity increases, we see that the total loss for different datasets drops to zero with different rates.
This reflects the fact that MNIST, Fashion-MNIST, CIFAR-10, and CIFAR-100 are in increasing order of difficulty.
\Cref{fig:complexity-structure-imagenet} illustrates the results of the same experiment on the (much larger) ImageNet dataset \cite{deng2009imagenet}.

\section{Discussion and Open Problems}
\label{sec:discussion}

The modern practice of Deep Learning is rich with success stories where high-dimensional parametric models are trained on large datasets, and then adapted (fine-tuned) to specific tasks.
This process requires considerable effort and expertise, and few tools are available to predict whether a given pre-trained model will succeed when fine-tuned for a different task.
We have started developing a language to reason about transfer learning in the abstract, and analytical tools that allow to predict the success of transfer learning.

The first step is to define the tasks and the space they live in.
We take a minimalistic approach, and identify a task with a finite dataset.
The second step is to endow the space of tasks with a metric.
This is non-trivial, since different datasets can have a different cardinality and dimension, and one needs to capture the fact that a simple task is usually ``closer'' to a complex one than vice-versa (in the sense that it is easier to fine-tune a model from a complex task to a simpler one).
Thus, a notion of complexity of a learning task needs to be defined first.

We introduce a notion of complexity and a notion of distance for learning tasks that are very general, and encompass well-known analogues in Kolmogorov complexity theory and Information Theory as special cases.
They espouse characteristics of Kolmogorov's framework (focusing on finite datasets rather than asymptotics) with the intuitive notions and ease of computation of Information Theory.
We use deep neural networks to compute information quantities.
On one hand, this provides a convenient way of instantiating our general theory.
On the other hand, it allows to measure the complexity of a Deep Network, and to reason about generalization (learning vs.\ memorization) using a non-asymptotic language, in terms of quantities that are measurable from finite datasets.

Our theory exposes connections between Deep Learning, Complexity Theory, and Information Theory, and PAC-Bayes theory, but more work is needed to fully develop these connections: The (static) distance we introduce only gives a lower bound to the feasibility of transfer learning with deep neural networks, and \emph{dynamics} plays an important role as well.
In particular, there are tasks that are very close, yet there is no likely path between them, so fine-tuning can fail.
This is observed broadly across different architectures and optimization schemes, but also across different species in biology, pointing to fundamental complexity and information phenomena yet to be fully unraveled.

\begin{figure}
    \centering
    \includegraphics[width=.35\linewidth, height=4.5cm]{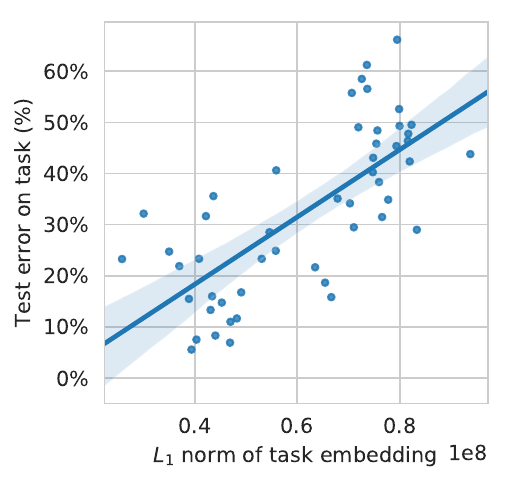}
    \caption{\textbf{Correlation between test error and the trace of the Fisher Information Matrix on several tasks} (reproduced from \cite{achille2019task2vec}).
    The Fisher Information Matrix emerges as a complexity measure when using an uninformative prior (\Cref{sec:complexity-measures}).
    This plot shows that that the FIM trace correlates with the error obtained on the task, and so the FIM is indeed a sensible measure of the complexity of a task.
    }
    \label{fig:fisher-norm}
\end{figure}

\subsubsection*{Acknowledgment}

Work supported by ONR N00014-19-1-2229 and ARO W911NF-17-1-0304.

\bibliographystyle{plain}
\bibliography{bibliography}

\begin{thebibliography}{10}

\bibitem{achille2019task2vec}
Alessandro Achille, Michael Lam, Rahul Tewari, Avinash Ravichandran, Subhransu
  Maji, Charless~C Fowlkes, Stefano Soatto, and Pietro Perona.
\newblock Task2vec: Task embedding for meta-learning.
\newblock In {\em Proceedings of the IEEE International Conference on Computer
  Vision}, pages 6430--6439, 2019.

\bibitem{achille2018dynamics}
Alessandro Achille, Glen Mbeng, and Stefano Soatto.
\newblock The dynamics of differential learning {I}: Information-dynamics and
  task reachability.
\newblock {\em arXiv preprint arXiv:1810.02440}, 2018.

\bibitem{achille2017critical}
Alessandro Achille, Matteo Rovere, and Stefano Soatto.
\newblock Critical learning periods in deep neural networks.
\newblock {\em Proceedings of the International Conference on Learning
  Representations}, 2019.

\bibitem{achille2017emergence}
Alessandro Achille and Stefano Soatto.
\newblock Emergence of invariance and disentanglement in deep representations.
\newblock {\em Journal of Machine Learning Research}, 19(50):1--34, 2018.

\bibitem{chaudhari2018stochastic}
Pratik Chaudhari and Stefano Soatto.
\newblock Stochastic gradient descent performs variational inference, converges
  to limit cycles for deep networks.
\newblock In {\em 2018 Information Theory and Applications Workshop (ITA)},
  pages 1--10. IEEE, 2018.

\bibitem{csurka2017domain}
Gabriela Csurka.
\newblock Domain adaptation for visual applications: A comprehensive survey.
\newblock {\em arXiv preprint arXiv:1702.05374}, 2017.

\bibitem{deng2009imagenet}
Jia Deng, Wei Dong, Richard Socher, Li-Jia Li, Kai Li, and Li~Fei-Fei.
\newblock Imagenet: A large-scale hierarchical image database.
\newblock In {\em 2009 IEEE conference on computer vision and pattern
  recognition}, pages 248--255. Ieee, 2009.

\bibitem{dziugaite2017computing}
Gintare~Karolina Dziugaite and Daniel~M Roy.
\newblock Computing nonvacuous generalization bounds for deep (stochastic)
  neural networks with many more parameters than training data.
\newblock {\em Proceedings of the 33rd Conference on Uncertainty in Artificial
  Intelligence}, 2017.

\bibitem{he2016deep}
Kaiming He, Xiangyu Zhang, Shaoqing Ren, and Jian Sun.
\newblock Deep residual learning for image recognition.
\newblock In {\em Proceedings of the IEEE conference on computer vision and
  pattern recognition}, pages 770--778, 2016.

\bibitem{hinton1993keeping}
Geoffrey Hinton and Drew Van~Camp.
\newblock Keeping neural networks simple by minimizing the description length
  of the weights.
\newblock In {\em Proceedings of the 6th Annual ACM Conference on Computational
  Learning Theory}, 1993.

\bibitem{lidata}
Ling Li.
\newblock {\em Data complexity in machine learning and novel classification
  algorithms}.
\newblock PhD thesis, California Institute of Technology, 2006.

\bibitem{lygeros1999controllers}
John Lygeros, Claire Tomlin, and Shankar Sastry.
\newblock Controllers for reachability specifications for hybrid systems.
\newblock {\em Automatica}, 35(3):349--370, 1999.

\bibitem{martens2014new}
James Martens.
\newblock New insights and perspectives on the natural gradient method.
\newblock {\em arXiv preprint arXiv:1412.1193}, 2014.

\bibitem{mcallester2013pac}
David McAllester.
\newblock A {PAC-B}ayesian tutorial with a dropout bound.
\newblock {\em arXiv preprint arXiv:1307.2118}, 2013.

\bibitem{springenberg2014striving}
Jost~Tobias Springenberg, Alexey Dosovitskiy, Thomas Brox, and Martin
  Riedmiller.
\newblock Striving for simplicity: The all convolutional net.
\newblock {\em arXiv preprint arXiv:1412.6806}, 2014.

\bibitem{tishby2000information}
Naftali Tishby, Fernando~C Pereira, and William Bialek.
\newblock The information bottleneck method.
\newblock {\em arXiv preprint physics/0004057}, 2000.

\bibitem{vereshchagin2004kolmogorov}
Nikolai~K Vereshchagin and Paul~MB Vit{\'a}nyi.
\newblock Kolmogorov's structure functions and model selection.
\newblock {\em IEEE Transactions on Information Theory}, 50(12):3265--3290,
  2004.

\bibitem{zamir2018taskonomy}
Amir~R Zamir, Alexander Sax, William Shen, Leonidas~J Guibas, Jitendra Malik,
  and Silvio Savarese.
\newblock Taskonomy: Disentangling task transfer learning.
\newblock In {\em Proceedings of the IEEE Conference on Computer Vision and
  Pattern Recognition}, pages 3712--3722, 2018.

\bibitem{zhang2016understanding}
Chiyuan Zhang, Samy Bengio, Moritz Hardt, Benjamin Recht, and Oriol Vinyals.
\newblock Understanding deep learning requires rethinking generalization.
\newblock In {\em Proceedings of the International Conference on Learning
  Representations}, 2017.

\end{thebibliography}

\appendix

\section{Proofs}

\begin{proof}[Proof of \Cref{prop:deterministic-complexity}]

(1) Let $p$ be such that $C_K(\D) = L_\D(p) + K(p)$. We can compress $\y$ using $p(\y|\x)$, for example with an algebraic code of length $L_\D(p) = -\log p(\y|\x)$.
A program that outputs $\y$ given $\x$ then only needs to encode the distribution $p(\y|\x)$ and the code for $\y$, requiring $L_\D(p) + K(p) + O(1)$ NATS,
so $K(\y|\x) \leq C(\D)$. For the opposite inequality, let
$h$ be the program that witness $K(\y|\x)$ and let $p_h(\y|\x) = \delta_{h(\x),\y}$. Then $K(p_h) = |h| = K(\y|\x)$ and $C_K(\D) \leq -\log p_h(\y|\x) + K(p_h) = K(\y|\x)$.

(2) Clearly, $C_K(\D) \leq C(\D)$ as $C(\D)$ minimizes over a smaller subset of distributions.
Since $C(\D)$ is permutation invariant, we have $C(\pi(\D)) = C(\D)$.
Hence, $C_K(\pi(D)) \leq C(\pi(\D)) = C(\D)$.

(3) Fix a function $f$ such that $K(f) \geq C$,
and let $h$ be a program for $f$.
Now, consider the dataset $\D={( x_i, f(x_i))}_{i=1}^N$, where
\[
	x_i = \begin{cases}
		\langle 0, i \rangle & \text{if $i < N$} \\
		\langle 1, h \rangle & \text{if $i = N$}.
	\end{cases}
\]
Here, the first bit is added in order to recognize the ``special'' data point.
We have $C_K(\pi(\D)) = O(1)$ for any permutation $\pi$, because the concatenation of the input data of $\D$ contains an encoding of $h$.
On the other hand, $C(\D) \geq C(\D') \geq C$, where $\D' = (x_i, f(x_i))_{i=1}^{N-1}$ is the datased obtained from $\D$ by removing the special data point, and $C(\D') = K(f) \geq C$ if $N$ is sufficiently large.

(4) Let $f\colon \X \to \Y$ be a function such that $f(x_i) = y_i$ for every $(x_i,y_i) \in \D$.
Consider the probability distribution $p(y | x)$ defined by $p(f(x) | x) = 1$ for every $x \in \X$. Then we have that $K(p) \leq K(f) + O(1)$, and $L_\D(p)=0$.
Choosing $f$ such that $K(f) = C_\tdet(\D)$, we get $C(\D) \leq L_\D(p) + K(p) \leq K(f) + O(1) = C_\tdet(\D) + O(1)$.

To prove the equality, let $h: \set{x_1,\ldots,x_N} \to N$ be the bijective function provided by the oracle. Now, create a list $A$ of codes so that $A[h(x_i)]$ contains the encoding of $y_i$ constructed using the distribution $p(y|x_i)$.
The length of the prefix code is $\lceil-\log p(y|x_i)\rceil+1$ (we need a prefix code so that we can concatenate all the codes).
Now, given the distribution $p$, we can construct a function $f$ such that $y_i=f(x_i)$  as follows: Given $x$, compute $h(x)$, read the code $A[h(x_i)]$, and decode it using the distribution $p(y|x_i)$ to obtain the correct $y_i$.
\end{proof}

\begin{lemma}
\label{lemma:MLE-cant-improve-asymptotically}
Fix a probability distribution $p(x,y)$ on $\X \times \Y$, and assume that $p(y|x)$ is computable.
Suppose that $\D$ is a collection of $N$ i.i.d.\ samples $(x_i, y_i) \sim p(x,y)$.
Let $A_x = \set{(x_i,y_i) \in D \mid x_i=x}$, and let $\hat{p}(y|x) = \frac{1}{|A_x|} \sum_{(x_i,y_i) \in A_x} \delta_{y_i, y}$ be the maximum likelihood estimation (MLE) of $p(y|x)$. Then, for every $\epsilon > 0$ there exists $c>0$ such that, in the limit $N \to \infty$ we have
\[
|L_\D(p) - L_\D(\hat{p})| < c
\]
with probability $> 1- \epsilon$.
\end{lemma}

\begin{proof}
    Consider $p(y|x)$ and $\hat{p}(y|x)$ as vectors of size $|\X \times \Y|$. Expand $L_\D(p)$ at $\hat{p}$:
    \[
    L_\D(p) - L_\D(\hat{p}) = \nabla_p L_\D(\hat{p}) (p - \hat{p}) + (p^*-\hat{p})^T \nabla^2_p L_\D(\hat{p}) (p^* - \hat{p}).
    \]
    Where $p^*$ is on the line connecting $p$ and $\hat{p}$. The first term is zero, since the MLE estimation $\hat{p}$ is by definition a minimum of $L_\D$. Now, recall that the MLE converges to the real distribution as $\sqrt{N} (\hat{p} - p) \to N(0, F(p)^{-1})$, where $F(p)$ is the Fisher Information Matrix computed in $p$, and $\nabla^2_p L_\D(\hat{p}) < N (\nabla^2_p  L_\D(p) + c')$. Let $a > 0$  be such that, with probability $1-\epsilon$, $\|\hat{p} - p\|<a$.
    Then,
    \[
    |L_\D(p) - L_\D(\hat{p}) | = | (p^*-\hat{p})^T \nabla^2_p L_\D(\hat{p}) (p^* - \hat{p}) | < a^2 \tr(\nabla^2_p  L_\D(p) + c')  =c. \qedhere
    \]
\end{proof}

\begin{proof}[Proof of \Cref{prop:iid-samples-complexity}]
(i) By Shannon's coding theorem, the expected value of $K(\y|\x)$ is at least $N\cdot H_p(y|x)$.
Then, the first inequality follows from part (1) of \Cref{prop:deterministic-complexity}.
If we use $p(y|x)$ itself in the definition of $C(\D)$, we obtain $C(\D) \leq L_\D(p) + K(p(y|x))$. The expected value of $L_\D(p)$ is $N \cdot H_p(y|x)$, so we obtain the second inequality.

(ii) We need to prove that, for any other distribution $p'$, we  have:
\[
L_\D(p') + K(p') > L_\D(p) + K(p).
\]
In fact, suppose that $L_\D(p') + K(p') \leq L_\D(p) + K(p)$. Then we have
\[
L_\D(p') - L_\D(p) \leq  K(p) - K(p').
\]
Notice that we can lower-bound the LHS using the MLE estimator $\hat{p}$, which by definition minimizes $L_\D$. By \Cref{lemma:MLE-cant-improve-asymptotically}, we have
$L_\D(\hat{p}) - L_\D(p) \geq -c$,
hence $K(p') \leq K(p) + c$.
By the central limit theorem, the LHS grows as $N \cdot \E_x[\KL{p(y|x)}{p'(y|x)}] + O(\sqrt{N})$, so we must have:
\[
\E_x[\KL{p(y|x)}{p'(y|x)}] \leq \frac{K(p)}{N} + O\Big(\frac{1}{\sqrt{N}}\Big).
\]
Therefore $\E_x[\KL{p(y|x)}{p'(y|x)}] \to 0$ as $N \to \infty$. On the other hand, we have that
\[ \E_x[\KL{p(y|x)}{p'(y|x)}] \geq k > 0 \]
for some $k$, since there are only finitely many distributions $p' \neq p$ such that $K(p') \leq K(p) + c$.
\end{proof}

\begin{proof}[Proof of \Cref{lemma:distance}]
	Positivity follows trivially by the definition.
	The second property is also easy: $d_\beta(\D \to \D) \leq \max_p K(p|p) = O(1)$.

	We now prove the triangle inequality. In what follows, $p_i$ and $\bar p_i$ are always $\beta$-minimal sufficient statistics of $\D_i$.
	Fix any $\bar p_1$.
	Choose $\bar p_2$ such that $K(\bar p_2|\bar p_1)$ is minimized.
	Then choose $\bar p_3$ such that $K(\bar p_3|\bar p_2)$ is minimized.
	Then
	\begin{align*}
		\min_{p_3} K(p_3 | \bar p_1) & \leq K(\bar p_3 | \bar p_1) \\[-0.2cm]
		& \leq K(\bar p_2 | \bar p_1) + K(\bar p_3 | \bar p_2) + O(1) \\
		& \leq d_\beta(\D_1 \to \D_2) + d_\beta(\D_2 \to \D_3) + O(1).
	\end{align*}
	This holds for every $\bar p_1$. Taking the maximum over $\bar p_1$, we obtain the desired result.
\end{proof}

\begin{proof}[Proof of \Cref{prop:concatenation}]
	For $(x,y) \in \X \times \Y$ and $i \in \{1,2\}$, define $p_{12}(y|x,i) = p_i(y|x)$.
	Then $\D_1 \sqcup \D_2$ is effectively obtained by sampling from $p_{12}$.
	By \Cref{prop:iid-samples-complexity}, with high probability and for $|\D_1|$ and $|\D_2|$ sufficiently large, we have that $p_1$ and $p_{12}$ are the unique minimal sufficient statistics for $\D_1$ and $\D_1 \sqcup \D_2$, respectively.
	By definition of $p_{12}$, we have that $K(p_1 | p_{12}) = O(1)$.
	Then $d_\beta(\D_1 \sqcup \D_2 \to \D_1) = O(1)$.
\end{proof}

\begin{proof}[Proof of \Cref{cor:union-distance}]
	One inequality follows by \Cref{lemma:distance} and \Cref{prop:concatenation}:
	\begin{align*}
		d_\beta(\D_1 \to \D_2) & \leq d_\beta(\D_1 \to \D_1 \sqcup \D_2) + d_\beta(\D_1 \sqcup \D_2 \to \D_2) + O(1) \\
		& = d_\beta(\D_1 \to \D_1 \sqcup \D_2) + O(1).
	\end{align*}

	To prove the other inequality, define $p_{12}$ as in the proof of \Cref{prop:concatenation}.
	Then
	\begin{align*}
		d_\beta(\D_1 \to \D_1 \sqcup \D_2) & = K(p_{12} | p_1) \\
		& \leq K(p_2|p_1) + O(1) \\
		& = d_\beta(\D_1 \to \D_2) + O(1). \qedhere
	\end{align*}
\end{proof}

\begin{proof}[Proof of \Cref{prop:kolmogorov-complexity-equivalence}]
	We have $\KL{Q(w|\D)}{P(w)} = -\log(e^{-K(w^*)} / Z) = K(w^*) + \log(Z)$.
\end{proof}

\begin{proof}[Proof of \Cref{prop:shannon-complexity-equivalence}]
	For a fixed training algorithm $A: \mathcal{D} \mapsto  Q(w|\D)$, we want to find the prior $P^*(w)$ that minimizes the expected complexity of the data:
	\begin{align*}
	P^*(w)
	&= \argmin_{P(w)} \E_\D[C(\D)] \\
	&= \argmin_{P(w)} \Big[\E_\D[L_\D(p_w(y|x))] + \E_\D[\KL{Q(w|\D)}{P(w)}] \Big]
	\end{align*}
	Notice that only the second term depends on $P(w)$. Let $Q(w) = \E_\D[Q(w|\D)]$ be the marginal distribution of $w$, averaged over all possible training datasets. We have
	\[
	\E_\D[\KL{Q(w|\D)}{P(w)}] = \E_\D[\KL{Q(w|\D)}{Q(w)}] + \E_\D[\KL{Q(w)}{P(w)}].
	\]
	Since the KL divergence is always positive, the optimal ``adapted'' prior is given by $P^*(w) = Q(w)$, i.e.\ the marginal distribution of $w$ over all datasets.
	Finally, by definition of Shannon's mutual information, we get
	\[ I(w; \D) = \KL{Q(w|\D) \, \pi(\D)}{Q(w) \,\pi(\D)} = \E_{\D \sim \pi(\D)}[\KL{Q(w|\D)}{Q(w)}]. \qedhere \]
\end{proof}

\begin{proof}[Proof of \Cref{prop:fisher-complexity-equivalence}]

Since both $P(w)$ and $Q(w|\D)$ are Gaussian distributions, the KL divergence can be written as
\[
\KL{Q(w|\D)}{P(w)} = \frac{1}{2} \bra{\frac{\norm{\mu}^2}{\lambda^2} + \frac{1}{\lambda^2} \tr(\Sigma) + k \log{\lambda^2} - \log|\Sigma| - k},
\]
where $k$ is the number of components of $w$.

Let $w^*$ be a local minimum of the cross-entropy loss $L_\D(p_w(y|x))$, and let $H$ be the Hessian of $L_\D(p_w(y|x))$ in $w^*$.
Set $\mu = w^*$.
Assuming that a quadratic approximation holds in a sufficiently large neighborhood, we obtain
\[
C_\beta(\D;P,Q) = L_\D(p_{w^*}(y|x)) + \frac12 \tr(H\cdot \Sigma)+\frac{\beta}{2} \bra{\frac{\norm{w^*}^2}{\lambda^2}+ \frac{1}{\lambda^2} \tr(\Sigma) + k \log \lambda^2 - \log |\Sigma|-k}.
\]
The gradient with respect to $\Sigma$ is
\[
\frac{\partial C_\beta(\D;P,Q)}{\partial \Sigma} = \frac12 \bra{H + \frac{\beta}{\lambda^2}I -\beta \Sigma^{-1}}^\top.
\]
Setting it to zero, we obtain the minimizer $\Sigma^* = \beta ( H + \frac{\beta}{\lambda^2} I)^{-1} $.

Recall that the Hessian of the cross-entropy loss coincides with the Fisher information matrix $F$ at $w^*$, because $w^*$ is a critical point \cite{martens2014new}.
Since $L_\D(p_w(y|x))$, and hence $H$, is not normalized by the number of samples $N$, the exact relation is $H = N\cdot F$.
Taking the limit for $\lambda \to \infty$, we obtain the desired result.
\end{proof}

\begin{proof}[Proof of \Cref{prop:local-learning-algorithm}]
If $Q_i$ is a global minimizer of $C_{\beta_i}$, then it is at distance at most $\epsilon$ from a global minimizer $Q_{i+1}$ of $C_{\beta_{i+1}}$. Therefore, for $\beta = \beta_{i+1}$, an $\epsilon$-local step from $Q_i$ reaches a global minimizer $Q_{i+1}$ of $C_{\beta_{i+1}}$.
By induction, the algorithm terminates at a global minimizer $\bar Q$ of $C_{\bar\beta}$.
\end{proof}

\end{document}